\documentclass[11pt]{article}

\usepackage{jmlr2e}

\usepackage[ruled,vlined]{algorithm2e}
\usepackage{amsmath, amssymb}
\usepackage{graphicx}
\usepackage{caption}
\usepackage{subcaption}
\usepackage{xcolor}
\usepackage{cancel}

\newcommand{\mb}{\mathbb}
\newcommand{\mc}{\mathcal}
\newcommand{\rar}{\rightarrow}

\newcommand{\bmat}[1]{\begin{bmatrix}#1\end{bmatrix}}
\newtheorem{assumption}{Assumption}
\newcommand{\spec}{\mathrm{spec}}
\newcommand{\diag}{\mathrm{diag}}

\newcommand{\Do}{\mathcal{D}_o}
\renewcommand{\Re}{\mathrm{Re}}
\newcommand{\SigmaK}[1]{\Sigma_{#1}}
\newcommand{\Ktempvar}[1]{\bar{K}_{#1}}

\newcommand{\Aclbar}{\bar{A}}
\newcommand{\vect}{\mathrm{vect}}
\newcommand{\nplayers}{N}
\newcommand{\Acl}[1]{A-\sum_{i=2}^\nplayers B_iK_i}
\newcommand{\nplayerstrategy}{K_1, \ldots, K_{\nplayers}}
\newcommand{\nplayerstrategyast}{K_1^\ast, \ldots, K_{\nplayers}^\ast}
\newcounter{smoothpropnum}
\newcounter{totalcntnum}

\usepackage{multirow}
 
 \usepackage{enumitem}

\usepackage{times}

\usepackage{amsmath, amssymb}
\usepackage{graphicx}
\usepackage{caption}
\usepackage{subcaption}
\usepackage{multirow}

\ShortHeadings{Policy-Gradient Algorithms in Linear Quadratic Games}{Mazumdar, Ratliff, Jordan, and Sastry}
\firstpageno{1}

\begin{document}

\title{Policy-Gradient Algorithms Have No Guarantees of Convergence in Linear Quadratic Games}
\author{\name Eric Mazumdar \email emazumdar@eecs.berkeley.edu \\ \addr Department of Electrical Engineering and Computer Science\\
       University of California, Berkeley, CA 
       \AND
       \name Lillian J. Ratliff \email ratliffl@uw.edu \\ \addr Department of Electrical and Computer Engineering\\
       University of Washington, Seattle, WA 
       \AND
       \name Michael I.\ Jordan \email jordan@cs.berkeley.edu \\ \addr Division of Computer Science and Department of Statistics\\
       University of California, Berkeley, CA 
       \AND
       \name S. Shankar Sastry \email sastry@coe.berkeley.edu \\ \addr Department of Electrical Engineering and Computer Science\\
       University of California, Berkeley, CA }

\editor{}

\maketitle

\begin{abstract}%

We show by counterexample that policy-gradient algorithms have no guarantees of even local
convergence to Nash equilibria in continuous action and state space multi-agent settings. To do so, we analyze gradient-play in $N$--player general-sum
linear quadratic games, a classic game setting which is recently emerging as a benchmark in the field of multi-agent learning.
In such games the state and action spaces are continuous and global Nash equilibria can be found be solving coupled Ricatti equations. Further, gradient-play in LQ games is equivalent to multi-agent policy-gradient. We first show that these games are surprisingly not convex games. Despite this, we are still able to show that the only critical points of the gradient dynamics are global Nash equilibria. We then give sufficient conditions under which policy-gradient will avoid the Nash equilibria, and generate a large number of general-sum linear quadratic games that satisfy these conditions. In such games we empirically observe the players converging to limit cycles for which the time average does not coincide with a Nash equilibrium. The existence of such games indicates that one of the most popular approaches to solving reinforcement learning problems in
the classic reinforcement learning setting has no local guarantee of convergence in multi-agent settings. Further, the ease with which we can generate these counterexamples suggests that such situations are not mere edge cases and are in fact quite common.

\end{abstract}

\section{Introduction}
Interest in multi-agent reinforcement learning  has seen a recent surge of late, and policy-gradient algorithms are championed due to their potential scalability. Indeed, recent impressive successes of multi-agent reinforcement learning have made use of policy optimization algorithms such as multi-agent actor-critic \citep{openAI_actcritic,deepmind_MAActorCritic,deepmind_CaptureFlag}, multi-agent proximal policy optimization \citep{openAI_MA_policygrad}, and even simple multi-agent policy-gradients \citep{deepmindMA_policygrad} in problems where the various agents have high-dimensional continuous state and action spaces like StarCraft II \citep{AlphaStar}.

Despite these successes, a theoretical understanding of these algorithms in multi-agent settings is still lacking. Missing perhaps, is a tractable yet sufficiently complex setting in which to study these algorithms. Recently, there has been much interest in analyzing the convergence and sample complexity of policy-gradient algorithms in the classic linear quadratic regulator (LQR) problem from optimal control \citep{KalmanLQR}. The LQR problem is a particularly apt setting to study the properties of reinforcement learning algorithms due to the existence of an optimal policy which is a linear function of the state and which can be found by solving a Ricatti equation. Indeed, the relative simplicity of the problem has allowed for new insights into the behavior of reinforcement learning algorithms in continuous action and state spaces \citep{LQRRecht,FazelLQR,wainwrightLQR}. 

An extension of the LQR problem to the setting with multiple agents, known as a \emph{linear quadratic (LQ) game}, has also been well studied in the literature on dynamic games and optimal control \citep{BasarOlsder}. As the name suggests, an LQ game is a setting in which multiple agents attempt to optimally control a shared linear dynamical system subject to quadratic costs. Since the players have their own costs, the notion of `optimality' in such games is a Nash equilibrium properties of which have been well analyzed in the literature \cite{scalarNash,allNash,Basar2,lukes}. 

Like LQR for the classical single-agent setting, LQ games are an appealing setting in which to analyze the behavior of multi-agent reinforcement learning algorithms in continuous action and state spaces since they admit global Nash equilibria in the space of linear feedback policies.  Moreover, these equilibria can be found by solving a coupled set of Ricatti equations. 
As such, LQ games are a natural benchmark problem on which to test policy-gradient algorithms in multi-agent settings. Furthermore, policy gradient methods open up the possibility to new scalable approaches to finding solutions to control problems even with constraints.  In the single-agent setting, it was recently shown that policy-gradient has global convergence guarantees for the LQR problem~\citep{FazelLQR}. These results have recently been extended to projected policy-gradient algorithms in zero-sum LQ games~\citep{basarLQ}.

\paragraph{Contributions.}  We present a \emph{negative} result, showing that policy-gradient in general-sum LQ games does not enjoy \emph{even local} convergence guarantees, unlike in  LQR and zero-sum LQ games. In particular, we show that, if each player randomly initializes their policy and then uses a policy-gradient algorithm, there exists an LQ game in which the players would almost surely avoid a Nash equilibrium. Further, our numerical experiments indicate that LQ games in which this occurs may be quite common. We also observe empirically that when players fail to converge to the Nash equilibrium they do converge to stable limit cycles. These cycles do not seem to have any readily apparent relationship to the Nash equilibria of the game. 

We note that non-convergence to Nash equilibria is not in itself a new phenomenon (see e.g. \cite{paper:zsg,daskalakisGANS,predictionlearning})  and that the existence of cycles in the dynamics of learning dynamics in games has also been repeatedly observed in various contexts \cite{paper:old,mertCycle,PapaGameDyn}.  However, we believe that such phenomena have not yet been shown to occur in the dynamics of multi-agent reinforcement learning algorithms in continuous action and state spaces. Since such algorithms have had such striking successes in recent years, we believe a theoretical understanding of their behaviors can lay the groundwork for the development of more efficient and theoretically sound  multi-agent learning algorithms.

\paragraph{Organization.}   Section~\ref{sec:prelims} introduces  $\nplayers$-player general-sum LQ games and presents previous results on the existence of the Nash equilibrium in such games. In Section~\ref{sec:thm}, we show that these games are \emph{not} convex games and that all the stationary points of the joint policy-gradient dynamics are Nash equilibria. Following this, we give sufficient conditions under which policy-gradient almost surely avoids a Nash equilibrium in Section~\ref{sec:nonconv}. Given these theoretical results, in Section~\ref{sec:num} we present empirical results demonstrating that a large number of 2-player LQ games satisfy these sufficient conditions.  Numerical experiments showing the existence of limit cycles in the gradient dynamics of general-sum LQ games are also presented. The paper is concluded with a discussion in Section~\ref{discussion}.

\section{Preliminaries}
\label{sec:prelims}



We consider $\nplayers$-player LQ games subject to a discrete-time dynamical
system defined by
\begin{align}
\textstyle z(t+1)=Az(t)+\sum_{i=1}^\nplayers B_iu_i(t) \ \; \ \  z(0)=z_0 \sim \Do,
\label{eq:updateLQR}
\end{align}
where $z(t) \in \mb{R}^m$ is the state at time $t$, $\Do$ is the initial
state distribution, and $u_i(t)\in\mb{R}^{d_i}$ is the control input of player $i \in {1,\ldots,\nplayers}$.  
For LQ games, it is known that under reasonable assumptions, linear feedback
policies for each player that constitute a Nash equilibrium  exist and are
unique if a set of coupled Ricatti equations admit a unique solution \citep{BasarOlsder}.
Thus, we consider that each player $i$ searches for a linear feedback policy of the form $u_i(t)=-K_iz(t)$
that minimizes their loss, where $K_i \in \mb{R}^{d_i\times m}$.  We use the notation $d=\sum_{i=1}^\nplayers d_i$ for the combined dimension of the players' parameterized policies.

As the name of the game implies, the players' loss functions are quadratic
functions given by
\[\textstyle f_i(u_1,\ldots,u_\nplayers)=\mb{E}_{z_0 \sim \Do} \left[\sum_{t=0}^\infty z(t)^TQ_iz(t)+u_{i}(t)^TR_iu_{i}(t)\right],\] 
where $Q_i$ and $R_i$ are the cost matrices for the state and
input, respectively. 
\begin{assumption}
 For each player $i\in\{1, \ldots, \nplayers\}$, the state and control cost matrices satisfy $Q_i \succ 0$ and $R_i \succ 0$.
\label{ass:coststatematrix}
\end{assumption}

We note that the players are coupled
through the dynamics since $z(t)$ is constrained to obey the update equation
given in~\eqref{eq:updateLQR}. 
We focus on a setting in which all players randomly initialize their strategy and then perform gradient descent simultaneously on their own cost functions with respect to their individual control inputs. That is, the players use policy-gradient algorithms of the following form:
\begin{align}
    K_{i,n+1}&=K_{i,n}-\gamma_i D_if_i(K_{1,n},\ldots,K_{\nplayers,n})
\end{align}
where $D_if_i(\cdot,\cdot)$ denotes the derivatives of $f_i$ with respect to the $i$--th argument, and $\{\gamma_i\}_{i=1}^\nplayers$ are the step-sizes of the players. We note that there is a slight abuse of notation here in the expression of $D_if_i$ as functions of the parameters $K_i$ as opposed to the control inputs $u_i$. To ensure there is no confusion between $t$ and $n$, we also point out that $n$ indexes the policy-gradient algorithm iterations while $t$ indexes the time of the dynamical system.

To simplify notation, define
\[\textstyle \SigmaK{K} =\mb{E}_{z_0\sim \Do}\left[\sum_{t=0}^\infty z(t)z(t)^T\right],\]
where we use the subscript notation to denote the dependence on the collection of controllers $K=(\nplayerstrategy)$. Define also the initial state covariance matrix
\[\textstyle \Sigma_0=\mb{E}_{z_0 \sim D_0}[z_0z_0^T].\]
Direct computation verifies that for player $i$, $D_if_i$ is given by:
\begin{align}
D_if_i(\nplayerstrategy)=2(R_{i}K_i-B_i^TP_i\Aclbar)\SigmaK{K},
\end{align}
where $\Aclbar=A-\sum_{i=1}^\nplayers B_iK_i$, is the closed--loop dynamics given all players' control inputs and, for given $(\nplayerstrategy)$, the matrix $P_i$ is the unique positive definite solution to the  Bellman equation:
 \begin{align}
P_i & = \Aclbar^TP_i\Aclbar+ K_i^TR_iK_i +Q_i, \ \
i\in\{1, \ldots, \nplayers\}.
\label{eq:Bellman}
\end{align}

Given that the players may have different control objectives and do not engage in coordination or cooperation, the best they can hope to achieve is a Nash equilibrium. 
\begin{definition}
A \emph{feedback Nash equilibrium} is a collection of policies $(\nplayerstrategyast)$ such that: 
\begin{equation*}
    f_i(K_1^*,\ldots,K_i^*,\ldots,K_\nplayers^*)\le f_i(K_1^*,\ldots,K_i,\ldots,K_\nplayers^*), \ \   \forall\ K_i\in \mb{R}^{ d_i \times m }.
\end{equation*}
for each $i \in \{1,\ldots,\nplayers\}$.
\label{def:feedbacknash}
\end{definition}
Under suitable assumptions on the cost matrices, the Nash equilibrium of an LQ game is known to exist in the space of linear policies~\cite{BasarOlsder,LyapIterCitation}. However, this Nash equilibrium may not be unique. To the best of our knowledge, there are no general set of conditions under which the Nash equilibrium is unique in general-sum LQ games outside of the scalar dynamics setting \cite{scalarNash}. There are, however, algebraic geometry methods to compute all Nash equilibria in LQ games \cite{allNash}. We make use of a simpler algorithm to find Nash equilibria which solves coupled Ricatti equations using the method of Lyapunov iterations. The method is outlined in \cite{LyapIterCitation} for continuous time LQ games, and an analogous procedure can be followed for discrete time. 
Convergence of this method requires the following assumption. 

\begin{assumption}
For at least one player $i \in \{1,\ldots,\nplayers\}$, $(A,B_i)$ is stabilizable. 
\label{ass:stable}
\end{assumption}

Assumption~\ref{ass:stable} is a necessary condition for the players to be able to stabilize the system. Indeed, the player's costs are finite only if the closed loop system $\Aclbar$ is asymptotically stable, meaning that $|\Re(\lambda)|<1$ for all $\lambda \in \spec(\Aclbar)$, where $\Re(\lambda)$ denotes the real part of $\lambda$ and $\spec(M)$ is the spectrum of a matrix $M$. 


\section{Analyzing the Optimization Landscape of LQ Games}
\label{sec:thm}

Having introduced the class of games we consider we now analyze the optimization landscape in general-sum LQ games. Letting $x=(\nplayerstrategy)$, the object of interest is the map $\omega:\mb{R}^{md}\rar \mb{R}^{md}$
defined as follows:
\[ \omega(x)=\bmat{D_1f_1(\nplayerstrategy)\\ \vdots \\ D_\nplayers f_\nplayers(\nplayerstrategy)}.\]
Note that $D_if_i=\partial f_i/\partial K_i$ has been converted to an $md_i$ dimensional vector and each $K_i$ has also been vectorized. This is a slight abuse of notation and throughout we treat the $K_i$'s as both vectors and matrices; in general, the shape should be clear from context, and otherwise we make comments where necessary to clarify.

Before analyzing the stationary points of policy-gradient in LQ games, we show that the class of LQ games we consider are \emph{not} convex games. This holds despite the linearity of the dynamics and the positive definiteness of the cost matrices. This fact makes the analysis of such games non-trivial since the lack of strong structural guarantees on the players' costs allows for non-trivial limiting behaviors like cycles, non-Nash equilibria, and chaos in the joint gradient dynamics. \cite{paper:old}.

\begin{proposition}
There exists a $\nplayers$-player LQ game satisfying assumptions \ref{ass:coststatematrix} and \ref{ass:stable} that is not a convex game.
\label{prop:convex}
\end{proposition}

\begin{proof}
 The proof of Proposition~\ref{prop:convex} follows directly from the non-convexity of the set of stabilizing policies for the single-agent LQR problem which was shown in \cite{FazelLQR}. Holding every other players' actions fixed,  a player $i$ is faced with a simple LQR problem. Since this problem is non-convex, LQ games are not convex games.
\end{proof}

In the absence of strong structural guarantees on the players' costs, simultaneous gradient-play in general-sum games can converge to strategies that are not Nash equilibria \citep{paper:old}. The following theorem shows that, despite the fact that LQ games are not convex for each player, such non-Nash equilibria cannot exist in the gradient dynamics of general-sum LQ games. Indeed, we show that a point $x$ is a critical point of the policy gradient dynamics in a $\nplayers$-player LQ game if and only if it is a Nash equilibrium. We note that critical points of gradient-play are strategies $x=(\nplayerstrategy)$ such that $\omega(x)=0$. Such points are of particular importance since a necessary condition for a point $x$ to be a Nash equilibrium is that it is a critical point. 

\begin{theorem}
Consider the set of stabilizing policies $x^*=(\nplayerstrategyast)$ such that $\SigmaK{K^\ast}>0$. $D_if_i(\nplayerstrategyast)=0$ for each $i\in\{1, \ldots, \nplayers\}$, if and only if $x^*$ is a Nash equilibrium.
\label{thm:cp}
\end{theorem}

\begin{proof}
 We prove the forward direction and show that if $D_if_i(x^\ast)=0$ for each $i\in\{1, \ldots, \nplayers\}$, then $x^*$ is a Nash equilibrium. We show this by contradiction. Suppose the claim does not hold so that $\Sigma_{K^\ast}>0$ and $D_if_i(\nplayerstrategyast)=0$ for each $i \in \{1,\ldots,\nplayers\}$, yet $(\nplayerstrategyast)$ is not a Nash equilibrium. That is, without loss of generality, there exists a $\Ktempvar{1}$ such that
 \[f_1(\Ktempvar{1}, K_2^*,\ldots,K^*_\nplayers)<f_1(\nplayerstrategyast).\]
 Now, fixing $(K_2^*,\ldots,K_\nplayers^*)$, player 1 can be seen as facing an LQR problem.  Indeed, letting $(K_2^*,\ldots,K_\nplayers^*)$ be fixed, player 1 aims to find a `best response' in the space of linear feedback policies of the form $u_1(t)=Kz(t)$ with $K\in \mb{R}^{d_i\times m}$ that minimizes $f_1(\cdot, K_2^*,\ldots,K_\nplayers^*)$ subject to the dynamics defined by
 \[\textstyle z(t+1)=\left(\Acl{2}\right)z(t)+B_1u_1(t).\]
 Note that this system is necessarily stabilizable since $\Aclbar$ is stable. Hence, the discrete algebraic Riccati equation for player 1's LQR problem has a positive definite solution $P$ such that $R_1+B_1^TPB_1>0$ since $R_1>0$ by assumption. Since $\SigmaK{K^\ast}>0$ and $D_1f_1(\nplayerstrategyast)=0$, applying Corollary 4 of \cite{FazelLQR}, we have that $K^*_1$ must be optimal for player 1's LQR problem so that
 \[f_1(\nplayerstrategyast)\leq f_1(K,K_2^*,\ldots,K^*_\nplayers), \ \ \forall\ K\in \mb{R}^{d_1 \times m}.\]
 In particular, the above inequality holds for $\Ktempvar{1}$, which leads to a contradiction. 
 
 To prove the reverse direction, we note that a necessary condition for a point $x$ to be a Nash equilibrium for each player, is that $D_if_i(x^\ast)=0$ for each $i\in\{1, \ldots, \nplayers\}$ \cite{ratliff:2013aa}. 
\end{proof}

Theorem~\ref{thm:cp} shows that, just as in the single-player LQR setting and zero-sum LQ games, the critical points of gradient-play in $\nplayers$--player general-sum LQ games are all Nash equilibria. We note that the condition $\SigmaK{K}>0$ can be satisfied by choosing an initial state distribution $\Do$ with a full-rank covariance matrix.

A simple consequence of Theorem~\ref{thm:cp} is that when the coupled Ricatti equations characterizing the Nash equilibria of the game have a unique positive definite solution and Assumptions~\ref{ass:coststatematrix} and~\ref{ass:stable} hold, the gradient dynamics admit a unique critical point. 

\begin{corollary}
Under Assumption~\ref{ass:coststatematrix} and~\ref{ass:stable}, if the coupled Ricatti equations admit a unique solution and $\Sigma_0\succ0$, then the map $\omega$ has a unique critical point.
\label{cor:unique}
\end{corollary}

 Given that the critical points of the gradient dynamics in LQ games are Nash equilibria, the aim is to show, via constructing counter-examples, that games  in which the gradient dynamics avoid the Nash equilibria do in fact exist. A sufficient condition for this would be to find a game in which gradient-play diverges from neighborhoods of Nash equilibria. 

It is demonstrated in \cite{paper:old} that there may be Nash equilibria that are not even \emph{locally attracting} under the gradient dynamics in $\nplayers$--player general-sum games in which the players' costs are sufficiently smooth (i.e., at least twice continuously differentiable). In games that admit such Nash equilibria, the agents could initialize arbitrarily close to the Nash equilibrium, simultaneously perform individual gradient descent with arbitrarily small step sizes, and still diverge. 

The class of $\nplayers$--player LQ games we consider does not, however, satisfy the smoothness assumptions necessary to simply invoke the results in \cite{paper:old}. Indeed, the cost functions are non-smooth and, in fact, are infinite whenever the players have strategies that do not stabilize the dynamics. Further, the set of stabilizing policies for a dynamical system is not even convex \citep{FazelLQR}. Despite these challenges, in the sequel we show that the negative convergence results in \cite{paper:old} extend to the general-sum LQ setting. In particular, we show that even with arbitrarily small step sizes, players using policy-gradient in LQ games may still diverge from neighborhoods of a Nash equilibrium.


\section{Sufficient Conditions for Policy-Gradient to Avoids Nash}
\label{sec:nonconv}

We now give sufficient conditions under which gradient-play has no guarantees of even \emph{local}, much less global, convergence to a Nash equilibrium.  Towards this end, we first show that $\omega$ is sufficiently smooth  on the set of stabilizing policies.

 Let $\mc{S}^{md}\subset \mb{R}^{md}$ be the subset of stabilizing $md$--dimensional matrices.
\setcounter{smoothpropnum}{\value{theorem}}
\begin{proposition}
Consider an $\nplayers$--player LQ game. The vector-valued map $\omega$ associated with the game is twice continuously differentiable on $\mc{S}^{md}$---i.e., $\omega\in C^2(\mc{S}^{md}, \mc{S}^{md})$.
\label{prop:smooth}
\end{proposition}

Using our notation, Lemma 6.5 in \cite{basarLQ} shows for two-player zero-sum LQ games that $(P_1,P_2)$, and $\SigmaK{K}$ are continuously differentiable with respect to $K_1$ and $K_2$ when $A-B_1K_1-B_2K_2$ is stable. This, in turn, implies that $\omega(K_1,K_2)$ is continuously differentiable with respect to $K_1$ and $K_2$ when the closed loop system $A-B_1K_1-B_2K_2$ is stable. The result follows by a straightforward application of the implicit function theorem~\citep{marsden:1988aa}. We utilize the same proof technique here in extending the result to $\nplayers$--player general-sum LQ games and, in fact, the proof implies that $\omega$ has even stronger regularity properties. Since the proof follows the same techniques as in \cite{basarLQ}, we defer it to Appendix~\ref{app:proofs}.


Given that $\omega$ is continuously differentiable over the set of stabilizing joint policies $(\nplayerstrategy)$, the following result gives sufficient conditions such that the set of initial conditions in a neighborhood of the Nash equilibrium from which gradient-play converges to the Nash equilibrium is of measure zero. This implies that the players will almost surely avoid the Nash equilibrium even if they randomly initialize in a uniformly small ball around it. 

Let the Jacobian of the vector field $\omega$ be denoted by $D\omega$. Given a critical point $x^\ast$, let $\lambda_j$  be the eigenvalues of $D\omega(x^\ast)$, for $j\in\{1, \ldots, md\}$, where $d=\sum_{i=1}^{n}d_i$.  Recall that the state $z(t)$ is dimension $m$.

\begin{theorem}
    Suppose that $\Sigma_0>0$. Consider any $\nplayers$--player LQ game satisfying Assumptions~\ref{ass:coststatematrix} and~\ref{ass:stable}  that admits a Nash equilibrium that is a saddle point of the policy-gradient dynamics---i.e., LQ games for which the Jacobian of $\omega$ evaluated at the Nash equilibrium $x^*=(\nplayerstrategyast)$ has eigenvalues $\lambda_j$ such that $\Re(\lambda_j)<0$ for $j\in\{1,\ldots,\ell\}$ and $\Re(\lambda_j)>0$ for $j\in\{\ell+1,\ldots,md\}$ for some $\ell$ such that $0<\ell<md$. Then there exists a neighborhood $U$ of $x^*$ such that policy-gradient converges on a set of measure zero.
   
    \label{thm:lqgame}
\end{theorem}
\begin{proof}
The proof is made up of three parts: (i) we show the existence of an open-convex neighborhood $U$ of $x^\ast$ on which $\omega$ is locally Lipschitz with constant $L$; (ii) we show that the map $g(x)=x-\Gamma\omega(x)$ is a diffeomorphism on $U$; and, (iii) we invoke the stable manifold theorem to show that the set of initializations in $U$ on which policy-gradient converges is measure zero.

\paragraph{(i) $\omega$ is locally Lipschitz.} Proposition~\ref{prop:smooth} shows that $\omega$ is continuously differentiable on the set of stabilizing policies $\mc{S}^{md}$. Given Assumptions~\ref{ass:coststatematrix} and \ref{ass:stable}, the Nash equilibrium exists and $x^* \in \mc{S}^{md}$. Thus, there must exist an open convex neighborhood $U$ of $x^*$ such that $||D\omega||_2<L$ for some $L>0$. 

\paragraph{(ii) $g$ is a diffeomorphism.}
By the preceding  argument, $\omega$ is locally Lipschitz on $U$ with Lipschitz constant $L$. Consider the policy-gradient algorithm with $\gamma_i<1/L$ for each $i\in\{1, \ldots, \nplayers\}$. 
        Let $\Gamma=\diag({\Gamma}_1, \ldots, {\Gamma}_\nplayers)$ where
        ${\Gamma}_i=\diag( (\gamma_i)_{j=1}^{md_i})$---that is,
        ${\Gamma}_i$ is an $md_i\times md_i$ diagonal matrix with $\gamma_i$
        repeated on the diagonal $md_i$ times. Now, we claim the mapping $g:\mb{R}^{md}\rar \mb{R}^{md}:x\mapsto x-\Gamma\omega(x)$ is a diffeomorphism on $U$. If we can
 show that $g$ is invertible on $U$ and a local diffeomorphism, then the claim follows.
Let us first prove that $g$ is invertible.

 Consider $x\neq y$ and suppose $g(y)=g(x)$ so that
    $y-x=\gamma\cdot (\omega(y)-\omega(x))$. 
       Since $\|\omega(y)-\omega(x)\|_2\leq L\|y-x\|_2$ on $U$, 
       $\|x-y\|_2\leq L\|\Gamma\|_2\|y-x\|_2<\|y-x\|_2$
       since $\|\Gamma\|_2=\max_i|\gamma_i|<1/L$.

       Now, observe that $Dg=I-\Gamma D\omega(x)$. If $Dg$ is invertible, then the
        implicit function theorem~\citep{marsden:1988aa} 
        implies that
        $g$ is a local diffeomorphism. Hence, it suffices to show that $\Gamma D\omega(x)$ does not have an eigenvalue equal to one. Indeed, letting $\rho(A)$ be
        the spectral radius of a matrix $A$, we know in general that
        $\rho(A)\leq \|A\|$ for any square matrix $A$ and induced operator norm
        $\|\cdot\|$ so that
         $\rho(\Gamma D\omega(x))\leq \|\Gamma D\omega(x)\|_2\leq
        \|\Gamma\|_2\sup_{x\in U}\|D\omega(x)\|_2<\max_i|\gamma_i|L<1$.
        Of course, the spectral radius is the maximum absolute value of the
        eigenvalues, so that the above implies that all eigenvalues of  $\Gamma
        D\omega(x))$ have absolute value less than one.

 Since $g$ is injective by the preceding argument, its inverse is well-defined and since $g$
        is a local diffeomorphism on $U$, it follows that $g^{-1}$ is
        smooth on $U$. Thus, $g$ is a diffeomorphism.

\paragraph{(iii) Local convergence occurs on a set of measure zero.} 
     Let $B$ be the open ball derived from Theorem~\ref{thm:centerstable} in Appendix~\ref{app:prelims}. 

            Starting from $x_0\in U$, if gradient-based learning
             converges to a strict saddle point, then there exists an $n_0$ such that $g^n(x_0)\in B$ for all $n\geq n_0$. 
             Applying Theorem~\ref{thm:centerstable} (Appendix~\ref{app:prelims}), we get that
             $g^n(x_0)\in W_{\text{loc}}^{cs}\cap B$. 
           Now, using the fact that $g$ is invertible, we can iteratively construct
             the sequence of sets defined by
             $W_1(x^\ast)=g^{-1}(W_{\text{loc}}^{cs}\cap B)\cap U$ and
             $W_{k+1}(x^\ast)=g^{-1}(W_k(x^\ast)\cap B)\cap U$. Then we have that $x_0\in
             W_n(x^\ast)$ for all $n\geq n_0$. The set $U_0=
             \cup_{k=1}^\infty W_k(x^\ast)$ contains all the initial points in $U$
             such that gradient-based learning converges to a strict saddle. 
     
             Since $x^\ast$ is a strict saddle, $I-\Gamma D\omega(x^\ast)$ has an
             eigenvalue greater than one. This implies that the
             co-dimension of the unstable manifold is strictly less than $md$ so that $\dim(W_{\text{loc}}^{cs})<md$.
            Hence,
             $W_{\text{loc}}^{cs}\cap B$ has Lebesgue measure zero in
             $\mb{R}^{md}$. 
             Using again that $g$ is a diffeomorphism, $g^{-1}\in C^1$ so that
             it is locally Lipschitz and locally Lipschitz maps are null-set
             preserving. Hence, $W_k(x^\ast)$ has measure zero for
             all $k$ by induction so that $U_0$ is a measure-zero set since
             it is a countable union of
             measure-zero sets.
\end{proof}

Theorem~\ref{thm:lqgame} gives sufficient conditions under which, with random initializations of $K_i$, policy-gradient methods would almost surely avoid the critical point. Let each players' initial strategy $K_{i,0}$ be sampled from a distribution $p_{i,0}$ for $i\in\{1,...,\nplayers\}$ , and let $p_0$ be the resulting the joint distribution of $(K_{1,0},\ldots,K_{\nplayers,0})$.

\begin{corollary}
Suppose $\Do$ is chosen such that $\Sigma_0\succ0$, and consider an $\nplayers$--player LQ game satisfying Assumptions~\ref{ass:coststatematrix} and ~\ref{ass:stable} in which there is a Nash equilibrium which is a saddle point of the policy-gradient dynamics. If each player $i\in\{1,\ldots,\nplayers\}$ performs policy-gradient with a random initial strategy $K_{i,0} \sim p_{i,0}$ such that the support of $p_0$ is $U$, they will almost surely avoid the Nash equilibrium.
\label{corr:almostsure}
\end{corollary}

Corollary~\ref{corr:almostsure} shows that even if the players randomly initialize in a neighborhood of a Nash equilibrium that is a saddle point of the joint gradient dynamics they will almost surely avoid it. The proof follows trivially from the fact that the set of initializations that converge to the Nash equilibrium is of measure zero in $U$.

In the next section, we generate a large number of LQ games that satisfy the conditions of Corollary~\ref{corr:almostsure}. Taken together, these theoretical and numerical results imply that policy-gradient algorithms have no guarantees of local, and consequently global, convergence in general-sum LQ games. 

\begin{remark} Theorem~\ref{thm:lqgame} gives us sufficient conditions under which policy-gradient in general-sum LQ games does not even have \emph{local convergence guarantees}, much less global convergence guarantees. We remark that this is very different from the single-player LQR setting, where policy-gradient will converge from any initialization in a neighborhood of the optimal solution \citep{FazelLQR}. In zero-sum LQ games, the structure of the game also precludes any Nash equilibrium from satisfying the conditions of  Theorem~\ref{thm:lqgame} \citep{paper:old}, meaning that local convergence is always guaranteed. In \cite{basarLQ}, the guarantee of local convergence is strengthened to that of global convergence for a class of projected policy-gradient algorithms in zero-sum LQ games.
\end{remark}

\setcounter{totalcntnum}{\value{theorem}}

\section{Generating Counterexamples}
\label{sec:num}


Since it is difficult to find a simple closed form for the Jacobian of $\omega$ due to the fact that the matrices $P_i$ implicitly depend on all the $K_i$, we perform random search to find instances of LQ games in which the Nash equilibrium is a strict saddle point of the gradient dynamics. For each LQ game we generate, we use the method of Lyapunov iterations to find a global Nash equilibrium of the LQ game and numerically approximate the Jacobian to machine precision. We then check whether the Nash equilibrium is a strict saddle. Surprisingly, such a simple search procedure finds a large number of LQ games in which policy-gradient avoids Nash equilibria.

For simplicity, we focus on two-player LQ games where $z\in \mb{R}^2$ and $d_1=d_2=1$. Thus, each player $i=1,2$ has two parameters to learn, which we denote $K_{i,j}$, $j=1,2$. 

In the remainder of this section, we detail our experimental setup and then present our findings.

\subsection{Experimental setup}
To search for examples of LQ games in which policy-gradient avoids Nash equilibria, we fix $B_1$, $Q_1$, and $R_1$ and parametrize $B_2$, $Q_2$, and $R_2$ by $b$, $q$, and $r$, respectively. For various values of the parameters $b$, $q$, and $r$, we uniformly sample $1000$ different dynamics matrices $A \in \mb{R}^{2 \times 2}$ such that $A,B_1,Q_1$ satisfies Assumption~\ref{ass:stable}. Then, for each of the $1000$ different LQ games we find the optimal feedback matrices $(K_1^*, K_2^*)$ using the method of Lyapunov iterations (i.e., a discrete time variant of the algorithm outlined in \cite{LyapIterCitation}), and then numerically approximate $D\omega(K_1^*,K_2^*)$ using auto-differentiation\footnote{We use auto-differentiation due to the fact that finding an analytical expression for $D\omega$ is unduly arduous even in low dimensions due to the dependence of $P_i$ and $\SigmaK{K_1,K_2}$ on $(K_1,K_2)$, both of which are implicitly defined.} tools and check its eigenvalues.

The exact values of the matrices are defined as follows: 
\begin{align*}
    A \in \mb{R}^{2\times 2}: a_{i,j}\sim {\mathrm{Uniform}}&(0,1)\ \quad  i,j=1,2,\\
    B_1=\begin{bmatrix}1\\1\end{bmatrix}, \
    B_2=\begin{bmatrix}b\\1\end{bmatrix}, \ Q_1=\begin{bmatrix}0.01 & 0\\0 &
        1\end{bmatrix}&, \ Q_2=\begin{bmatrix}1 & 0\\0 & q \end{bmatrix}, R_1=0.01, \ R_2=r.
\end{align*}

\subsection{Numerical results}

Using the setup outlined in the previous section we randomly generated LQ games to search for counterexamples. We first present results that show that these counterexamples may be quite common. We then use policy-gradient in two of the LQ games we generated and highlight the existence of limit cycles and the fact that the players' time-averaged strategies do not converge to the Nash equilibrium. 

\begin{figure*}[ht]
\center    
      \includegraphics[width=\linewidth]{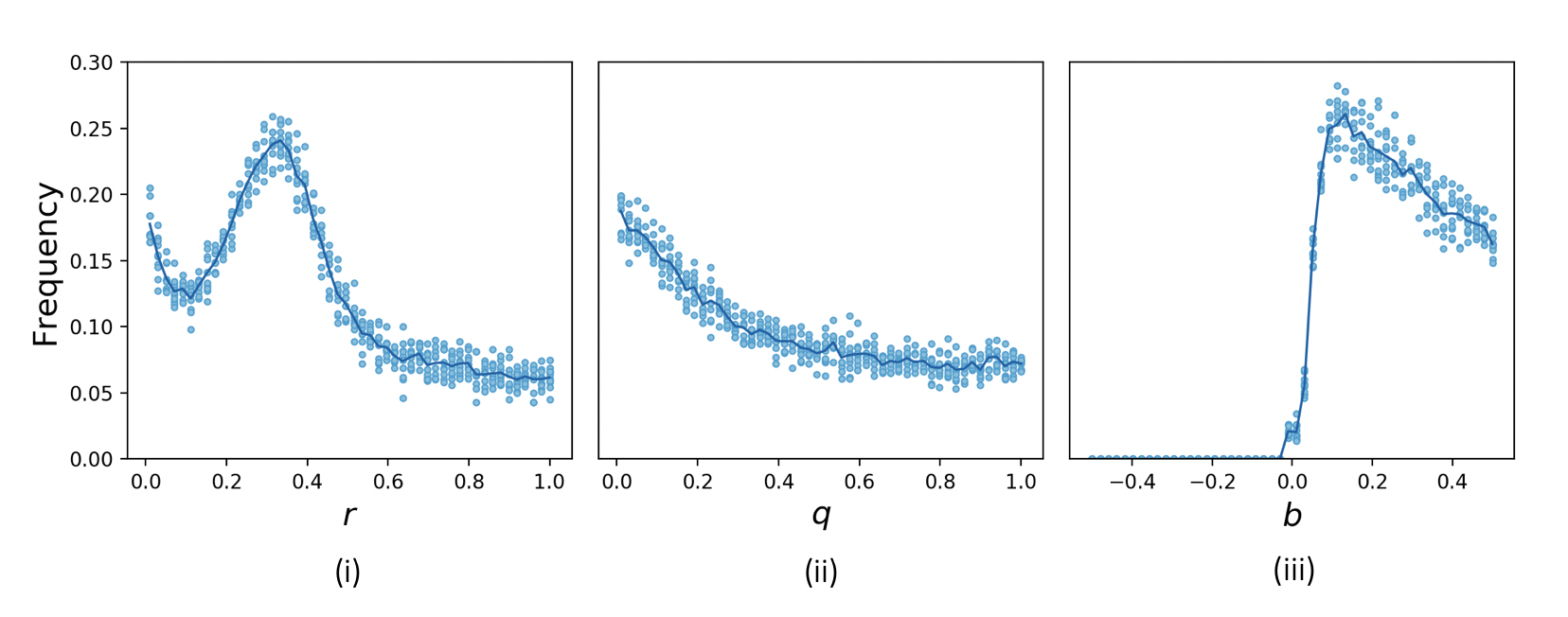}
        \caption{Frequency (out of 1000) of randomly sampled LQ games with global Nash equilibria that are avoided by policy-gradient. Each point represents, for the given parameter value, the frequency of such games out of 1000 randomly sampled $A$ matrices. The solid line shows the average frequency of these games.   (i) $r$ is varied in $(0,1)$, $b=0$, $q=0.01$. (ii) $q$ is varied in $(0,1)$, $b=0$, $r=0.1$. (iii) $b$ is varied in $(-0.5,0.5)$, $q=0.01$, $r=0.1$.}
        \label{fig:lqr}
\end{figure*}

\paragraph{Avoidance of Nash in a nontrivial class of LQ games.} 

As can be seen in Figure~\ref{fig:lqr}, across the different parameter values we considered, we found that anywhere from $0\%$ to $25\%$ of randomly sampled LQ games, had Nash equilibria that are strict saddle points of the gradient dynamics. Therefore, in up to $25\%$ of the LQ games we generated policy-gradient would almost surely avoid a Nash solution. Of particular interest, for all values of $q$ and $r$ that we tested, when $b=0$ at least $5\%$ of the LQ games had a global Nash equilibrium with the strict saddle property. 

\begin{figure}[h]
\center    
      \includegraphics[width=\textwidth]{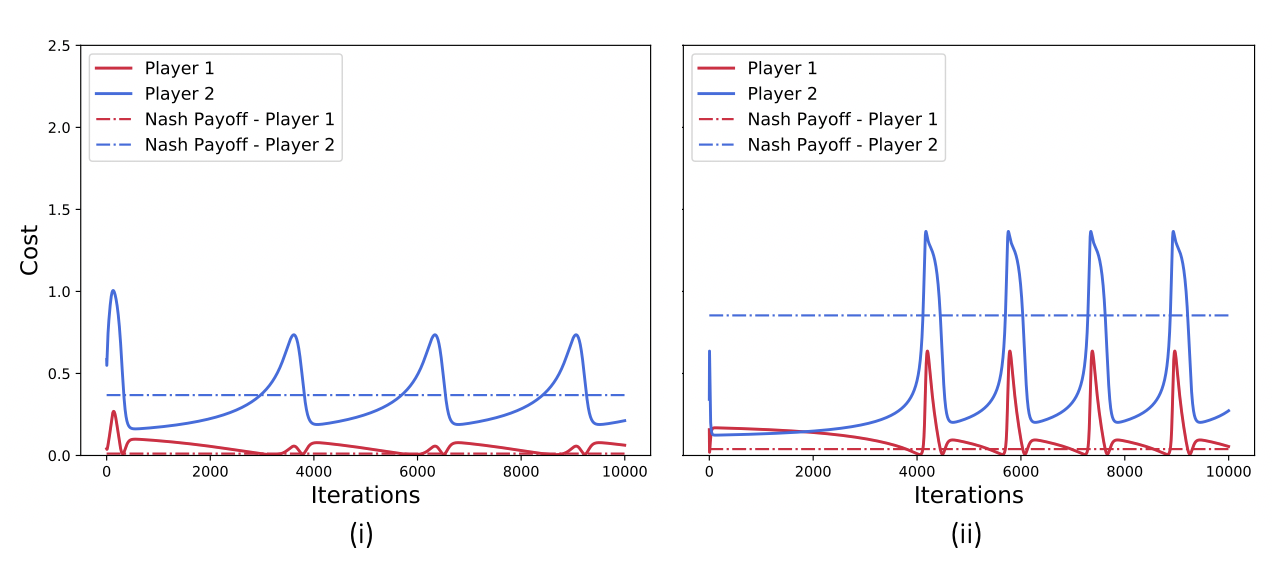}
        \caption{Payoffs of the two players in two general-sum LQ game where there is a Nash equilibrium that is avoided by the gradient dynamics. We observe empirically that in both games the two players diverge from the local Nash equilibrium and converge to a limit cycle around the Nash equilibrium.}
        \label{fig:path}
\end{figure}

These empirical observations imply that policy-gradient in competitive settings, even in the relatively straightforward setting of linear dynamics, linear policies, and quadratic costs, could fail to converge to a Nash equilibrium in up to one out of four such problems. This suggests that for more complicated cost functions, policy classes, and dynamics, Nash equilibria may often be avoided by policy-gradient. 

We remark that each point in Figure~\ref{fig:lqr} represents the number of counterexamples found (out of $1000$) for each parameter value, meaning that for $r\approx0.35, b=0,$ and $q=0.01$ we were able to consistently generate around $250$ different examples of games where policy-gradient almost surely avoids the only stationary point of the dynamics.

 Note also that we were unable to find any counterexamples when  $b$ was varied in $(-0.5,0.5)$ and $q=0.01$, $r=0.1$. This suggests that depending on the structure of the dynamical system it may be possible to give stronger convergence guarantees.

\paragraph{Convergence to Cycles.} Figures~\ref{fig:path}--\ref{fig:params2} show the payoffs and parameter values of the two players when they use policy-gradient in two general-sum LQ games we identified as being counterexamples for convergence to the Nash equilibrium. 

In the two games, we initialize both players in a ball of radius $0.25$ around their Nash equilibrium strategies and let them perform policy-gradient with step size $0.05$. We observe that in both games the players diverge from the Nash equilibrium and converge to limit cycles.

For the two games in Figures~\ref{fig:path}--\ref{fig:avg}, the game parameters are such that $b=0$, $r=0.01$, and  $q=0.147$.
The two $A$ matrices are defined as follows:
\begin{align}
\text{(i):} \ \ &A=\bmat{0.588 & 0.028 \\ 0.570 & 0.056}, \quad
    \text{(ii):} \ \ A=\bmat{0.511 & 0.064 \\ 0.533 & 0.993}.  
    \label{eq:paramsAii}
\end{align}

We also chose the initial state distribution to be $[1,1]^T$ or $[1,1.1]^T$ with probability 0.5 each.


\begin{figure*}[h]
\center    
      \includegraphics[width=\textwidth]{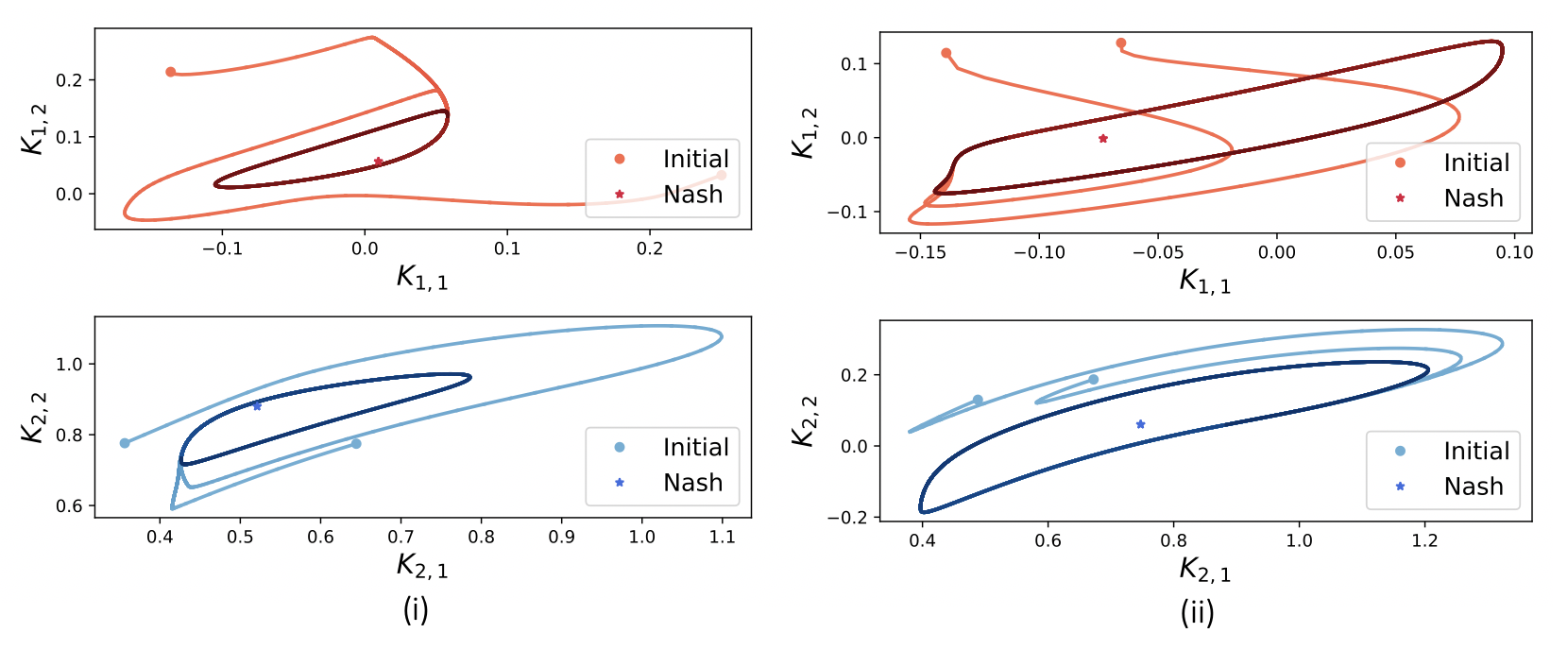}
        \caption{Parameter values of the two players in two general-sum LQ game where the Nash equilibrium is avoided by the gradient dynamics. We empirically observe in both games described in \eqref{eq:paramsAii} that players converge to the same cycle from different initializations. Time is shown by the progressive darkening of the players' strategies.}
        \label{fig:params2}
\end{figure*}

The eigenvalues of the corresponding game Jacobian $D\omega$ evaluated at the Nash equilibrium are as follows:
\begin{align*}
    \text{(i):} \ \ &\spec(D\omega(K_1^*,K_2^*))=\{10.88  ,  2.02  , -0.21 , -0.06 \} \\ 
    \text{(ii):} \ \ &\spec(D\omega(K_1^*,K_2^*))= \{9.76, 0.54,-0.01+0.08j, -0.01-0.08j\}.
\end{align*}
Thus, these games do satisfy the conditions of Corollary~\ref{corr:almostsure} for the avoidance of Nash equilibria.
We conclude this section by noting that, as shown in Figure~\ref{fig:avg}, the players' average payoffs do not necessarily converge to the Nash equilibrium payoffs. 

\begin{figure*}[t]
\center    
      \includegraphics[width=1.0\linewidth]{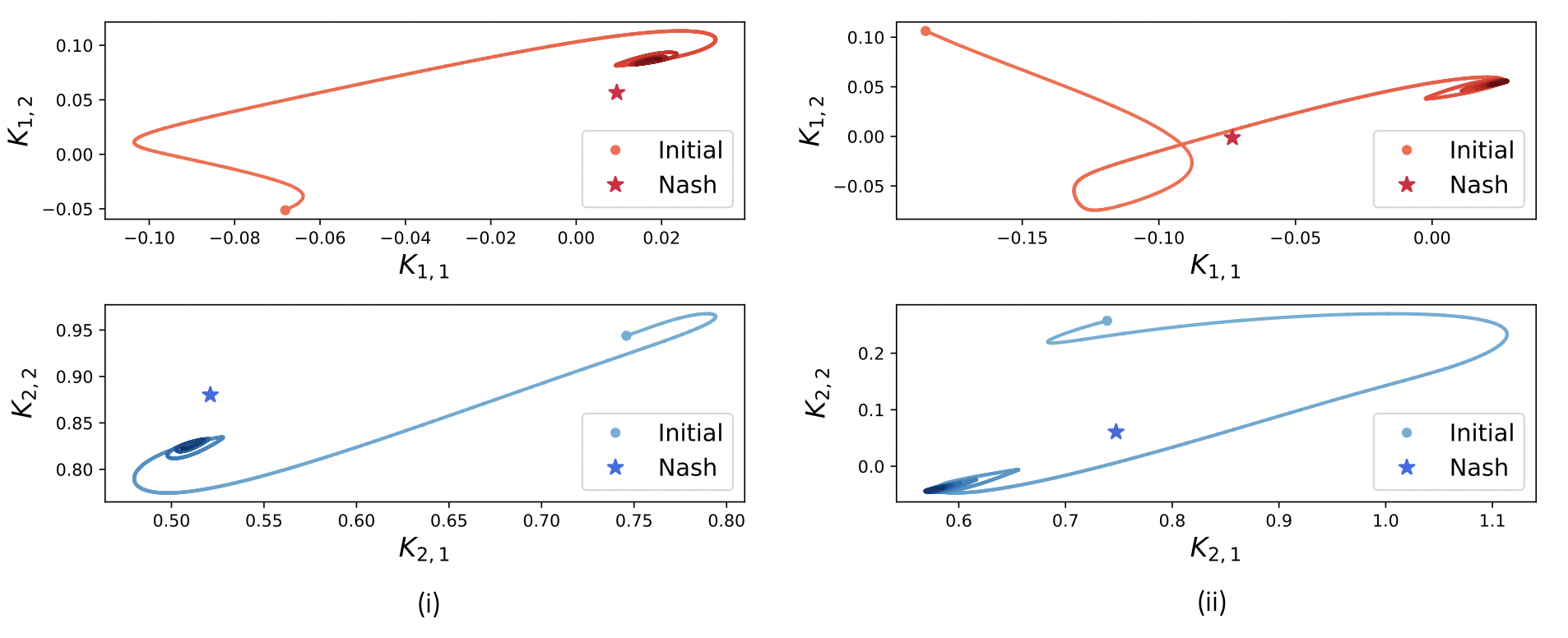}
        \caption{Time average parameter values of the two players in the general-sum LQ game with dynamics given in \eqref{eq:paramsAii}. We empirically observe that in both games the players' time average strategy does not converge to the Nash equilibrium strategy. Time is shown by progressive darkening of the players' strategies. }
        \label{fig:avg}
\end{figure*}

\section{Discussion}
\label{discussion}

We have shown that in the relatively straightforward setting of $\nplayers$--player LQ games, agents performing policy-gradient have no guarantees of local, and therefore global, convergence to the Nash equilibria of the game even if they randomly initialize their first policies in a small neighborhood of the Nash equilibrium. Since we also showed that the Nash equilibria are the only critical points of the gradient dynamics, this means that, for this class of games, policy-gradient algorithms may have no guarantees of convergence to \emph{any} set of stationary policies. 

Since linear dynamics, quadratic costs, and linear policies are a relatively simple setup compared to many recent deep multi-agent reinforcement learning problems \citep{openAI_MA_policygrad,deepmind_CaptureFlag}, we believe that the issues of non-convergence are likely to be present in more complex scenarios involving more complex dynamics and parametrizations of the policies. This can be viewed as a cautionary note, but it also suggests that the algorithms that have yielded impressive results in multi-agent settings can be further improved by leveraging the underlying game-theoretic structure.

We remark that we only analyzed the deterministic policy-gradient setting, though the findings extend to settings in which players construct unbiased estimates of their gradients \citep{sutton:2017aa} and even actor-critic methods \citep{deepmind_MAActorCritic}. Indeed all of these algorithms will suffer the same problems since they all seek to track the same limiting continuous-time dynamical system \citep{paper:old}. 

Our numerical experiments also highlight the existence of limit cycles in the policy-gradient dynamics. Unlike in classical optimization settings in which oscillations are normally caused by the choice of step sizes, the cycles we highlight are behaviors that can occur even with arbitrarily small step sizes. They are a fundamental feature of learning in multi-agent settings and have been observed in the dynamics of many learning algorithms \citep{paper:old,PapaGameDyn,cycleHommes,mertCycle}. We remark, however, that there is no obvious link between the limit cycles that arise in the gradient dynamics of the LQ games and the Nash equilibrium of the game. Indeed, unlike with other game dynamics in more simple games, such as the well-studied replicator dynamics in  bilinear games \citep{mertCycle}  or multiplicative weights in rock-paper-scissors \citep{cycleHommes}, the time average of the players' strategies does not coincide with the Nash equilibrium. This may be due to the fact that the Nash equilibrium is a saddle point of the gradient dynamics and not simply marginally stable, though the issue warrants further investigation. 

This paper highlights how algorithms developed for classical optimization or single-agent optimal control settings may not behave as expected in multi-agent and competitive environments. Algorithms and approaches that have provable convergence guarantees and performance in competitive settings, while retaining the scalability and ease of implementation of simple policy-gradient methods, are therefore a crucial and promising open area of research.


\appendix
\section{Proofs of Auxiliary Results}
\label{app:proofs}
{\setcounter{theorem}{\value{smoothpropnum}}
\begin{proposition}
Consider an $\nplayers$--player LQ game. The vector-valued map $\omega$ twice continuously differentiable on $\mc{S}^{md}$; i.e., $\omega\in C^2(\mc{S}^{md}, \mc{S}^{md})$.
\end{proposition}}
\begin{proof} 
 Following the proof technique of \cite{basarLQ}, we show the regularity of $\omega$ using the implicit function theorem~\citep{marsden:1988aa}. In particular, we show that $\SigmaK{K}=\mb{E}_{z_0\sim\Do}\left[\sum_{t=0}^{\infty}z(t)z(t)^T\right]$ and $P_i$ for $i\in\{1, \ldots, \nplayers\}$ are $C^1$ with respect to each $K_i$ on the space of stabilizing matrices.
 
For any stabilizing $(\nplayerstrategy)$, $\SigmaK{K}$ is the unique solution to the following discrete-time Lyapunov equation:
 \begin{equation}\Aclbar\SigmaK{K}\Aclbar^T+\Sigma_0=\SigmaK{K}, \label{eq:sigmaeq}\end{equation}
 where $\Sigma_0=\mb{E}_{z_0\sim\Do}[z(0)z(0)^T]>0$ and $\Aclbar=A-\sum_{i=1}^{\nplayers}B_iK_i$. Both sides of this expression can be vectorized. Indeed, using the same notation as in \cite{basarLQ}, let $\vect(\cdot)$ be the map that vectorizes its argument and let $\Psi:\mb{R}^{m^2}\times\mb{R}^{d_1\times m}\times\cdots\times\mb{R}^{d_\nplayers\times m}\rar \mb{R}^{m^2}$ be defined by
 \[\Psi(\vect(\SigmaK{K}), \nplayerstrategy)=\left[\Aclbar \otimes \Aclbar \right]\cdot \vect(\SigmaK{K})+\vect(\Sigma_0).\]
 Then, \eqref{eq:sigmaeq} can be written as
\begin{align*}
    F(\vect(\SigmaK{K}),\nplayerstrategy)&=\Psi(\vect(\SigmaK{K}),\nplayerstrategy)-\vect(\SigmaK{K})\\
    &=0.
\end{align*}
 The map $F$ implicitly defines $\SigmaK{K}$. Moreover, letting $I$ denote the appropriately sized identity matrix, we have that
 \[\frac{\partial F(\vect(\SigmaK{K}),\nplayerstrategy)}{\partial \vect^T(\SigmaK{K})}=\left[\Aclbar \otimes \Aclbar\right]-I.\]
For stabilizing $(\nplayerstrategy)$, this matrix is an isomorphism since $\spec(\Aclbar)$ is inside the unit circle. Thus, using the implicit function theorem, we conclude that $\vect(\SigmaK{K})\in C^1$. As noted in \cite{basarLQ}, the proof for each $P_i$, $i\in \{1, \ldots, \nplayers\}$ is completely analogous. Since $\SigmaK{K}$ and $P_i$ are $C^1$ and $\omega$ is linear in these terms, the result of the proposition follows. 
\end{proof}

\section{Additional Mathematical Preliminaries and Results}
The following theorem is the celebrated center manifold theorem from geometry. We utilize it in showing avoidance of saddle point equilibria of the dynamics.  
\label{app:prelims}
\setcounter{theorem}{\value{totalcntnum}}
\begin{theorem}[{Stable Manifold Theorem~\cite[Thm.~III.7]{shub:1978aa},
\cite{smale:1967aa}}]
    Let $x_0$ be a fixed point for the $C^r$ local diffeomorphism $\phi:U\rar
    \mb{R}^n$ where $U\subset \mb{R}^n$ is an open neighborhood of $x_0$ in
    $\mb{R}^n$ and $r\geq 1$. Let $E^s\oplus E^c\oplus E^u$ be the invariant
    splitting of $\mb{R}^n$ into generalized eigenspaces of $D\phi(x_0)$
    corresponding to eigenvalues of absolute value less than one, equal to one,
    and greater than one. To the $D\phi(x_0)$ invariant subspace $E^s\oplus E^c$
    there is an associated local $\phi$--invariant $C^r$ embedded disc
    $W_{\text{loc}}^{cs}$ called the local stable center manifold of dimension $\dim(E^s\oplus E^c)$ and ball $B$ around
    $x_0$ such that
    $\phi(W_{\text{loc}}^{cs})\cap B\subset W_{\text{loc}}^{cs}$, and if $\phi^n(x)\in
    B$ for all $n\geq 0$, then $x\in W_{\text{loc}}^{sc}$.
    \label{thm:centerstable}
\end{theorem}

\bibliography{JMLR}

\end{document}